\documentclass{article}

\usepackage{PRIMEarxiv}

\usepackage[utf8]{inputenc} 
\usepackage[T1]{fontenc}    
\usepackage{hyperref}       
\usepackage{url}            
\usepackage{booktabs}       
\usepackage{amsfonts}       
\usepackage{nicefrac}       
\usepackage{microtype}      
\usepackage{lipsum}
\usepackage{fancyhdr}       
\usepackage{graphicx}       
\usepackage{algorithm}
\usepackage{amsthm}
\usepackage{float}
\usepackage{algpseudocode}
\graphicspath{{media/}}     
\graphicspath{ {./images/} }
\newtheorem{df}{Definition}
\newtheorem{theorem}{Theorem}
\newtheorem{lm}{Lemma}
  
\title{Schema matching using Gaussian mixture models with\,Wasserstein distance
}

\author{
  Mateusz Przyborowski \\
  University of Warsaw, Warsaw, Poland \\
  QED Software, Warsaw, Poland \\
  \texttt{ml.przyborowsk@uw.edu.pl} \\
  \And
  Mateusz Pabiś \\
QED Software, Warsaw, Poland \\
  \texttt{mateusz.pabis@qed.pl} \\
  \And
  Andrzej Janusz \\
  University of Warsaw, Warsaw, Poland \\
  QED Software, Warsaw, Poland \\
  \texttt{andrzej.janusz@qed.pl} \\
  \And
  Dominik Ślęzak \\
  University of Warsaw, Warsaw, Poland \\
  QED Software, Warsaw, Poland \\
  \texttt{dominik.slezak@qed.pl} \\
}

\fancypagestyle{firststyle}
{
   \fancyhf{}
   \fancyfoot[L]{\footnotesize This research was co-funded by Smart Growth Operational Programme 2014-2020, financed by European Regional Development Fund, in frame of project POIR.01.01.01-00-0570/19, operated by National Centre for Research and Development in Poland.}
   
}

\begin{document}
\maketitle
\thispagestyle{firststyle}

\begin{abstract}
Gaussian mixture models find their place as a powerful tool, mostly in the clustering problem, but with proper preparation also in feature extraction, pattern recognition, image segmentation and in general machine learning. When faced with the problem of schema matching, different mixture models computed on different pieces of data can maintain crucial information about the structure of the dataset. In order to measure or compare results from mixture models, the Wasserstein distance can be very useful, however it is not easy to calculate for mixture distributions. In this paper we derive one of possible approximations for the Wasserstein distance between Gaussian mixture models and reduce it to linear problem. Furthermore, application examples concerning real world data are shown.
\end{abstract}


\section{Introduction}
Mixture model is a probabilistic model that is able to infer subpopulations from total population without additional information (within the paradigm of unsupervised learning). Mixture models closely correspond to the mixture distributions of the probabilistic distributions of observations. In general, in the structure of mixture model, we make assumptions over latent variables that evaluate membership of each observation. Given the dataset, we can assume that it is a sample and then mixture model can estimate the parameters of the probability distributions that created points of this dataset, as well as assign each observation vector of probabilities indicating the original distribution. \\
Comparing different mixture models can be considered a generalization of the problem of comparing different distributions. From the viewpoint of optimal transport theory, the Wasserstein distance is an important method for measuring similarities and the maintenance of the explainable nature of mixture models. \\
In this paper we derive one of possible approximations of Wasserstein distances computed between mixture models, which may be reduced to linear optimization problem, and we present examples of usage

\section{Related work}
Gaussian mixture models with Wasserstein distance find their place in many areas of machine learning. In case of generative networks\cite{gaujac}, the use of Wasserstein distance has been proved to model more complex distributions. Autoencoder architectures equipped with Wasserstein distance (WAE), unlike variational autonencoders (VAE), allow to use deterministic mapping to a latent space\cite{smieja}. In image processing, Gaussian mixture models equipped with Wasserstein distance proved to be useful in tasks of color transfer and texture synthesis\cite{delon}. When dealing with heterogeneous data, mixture models have the advantage of simplicity and Wasserstein distance provides a suitable convergence rate\cite{ozkan}. Moreover, Wasserstein distance holds an important place in optimal transport theory\cite{peyr}\cite{takatsu}. 

\section{Problem formulation} 
Let $p(x|z)$ be a probability distribution of the given data with unknown vector of parameters $z$. Modeling the data using statistics and machine learning comes to modeling probability distribution. In real world applications, data is usually composed of multiple different probability distributions. Hence comes an elementary idea of modeling the data using mixture model, where each observation is assigned a probability of originating from the given probability distribution. The problem of choosing types of probability distributions for each component is usually skipped by assuming normality (Gaussian) of individual components, as normal distribution has important probabilistic properties. This approach is focused on a general summary of the very origin of the data, therefore its applications are widespread:
\begin{enumerate}
\item in cluster analysis, Gaussian mixture models (GMM) may be seen as an extension to K-means algorithm, yielding additional information about given observations;
\item in supervised learning, associating a type of labels from training data with one or more components may give us a similarity function between observations, based on whether they originate from the same probability distribution;
\item in natural language processing, distribution of words in documents can be modelled as mixture of different categorical distributions.
\end{enumerate}
\subsection{Big data}
Nowadays dealing with the big data is a popular issue. While focusing on a big volume of moderately dimensional data, mixture models can help with summary of the most common type of observations. Suppose that size of the data makes it impractical to repetitively perform calculations using the entire data. If we could summarize the data by creating representations, which allow to maintain most important features of the data, as well as to perform calculations yielding approximate but much faster solutions, we would save a lot of computing power and time in practical applications. Mixture models may be considered as one of such approaches, in which the data representation is made of the components understood as parameters of probability distributions. 
The mixture model of a given dataset is itself an approximation of the underlying probability distribution. While it gives a way to compare different observations from the same dataset, one may think about comparing different representations, i.e. different mixture models. Suppose that we split a labeled dataset into datasets based on label, then compute mixture model for each of such datasets. Under the assumption that different labels indicate a different distribution of features, comparing the mixture models allows us to conclude that two datasets originate from similar sources. This problem is more widely known as schema matching problem and is a common task in data integration and database management. 
\subsection{Comparing mixture models}
Mixture models, by the very way they are calculated, are based on the values of many observations. The only difference between the resulting models must be a manifestation of the different values of the respective observations. This interpretation yields a corollary that the difference between models could be measured by how much and how many observations making up one mixture model must be transformed in order for the resulting model to be more similar to the one with which it is compared. This intuition is realized in the Wasserstein metric, where the distance between two probability distributions is the amount of ''work'' that needs to be done in order to transform one distribution into another. Further explanation is provided in the following sections. \\
Gaussian mixture models allow us to summarize large datasets, while Wasserstein distance makes a tool for comparing different representations. 

\section{Gaussian mixture models}
Henceforth we will focus on the Gaussian mixture models, i.e. mixture models only with normal components.
\begin{df}
Let $n\in \mathbb{N}_+$, $ w= [w_1,...,w_n] $ s.t. $\forall i\leq n \quad 0\leq w_i\leq 1$ and $\sum_{i=1}^n w_i = 1$. A Gaussian mixture model of size $n$ components is a probability distribution $p$ defined as:
$$p=\sum_{i=1}^n w_i\mathcal{N}(\mu_i,\Sigma_i)$$
where $\mathcal{N}(\mu_i,\Sigma_i)$ is a normal distribution with $\mu_i$the vector of mean and $\Sigma_i$ the covariance matrix as parameters corresponding to the $i$-th component. 
\end{df}
Thereafter, if not stated otherwise, we will assume that distributions are defined over $\mathbb{R}^d$ with the dimension $d$. For the sake of simplicity, if it is not necessary we will omit the dimension. \\
Fitting Gaussian mixture model to a given data is a task of finding appropriate values for parameters $\{w_i, \mu_i, \Sigma_i\}_{i=1}^n$ s.t. the resulting model describes the dataset. Using maximum likelihood estimation (MLE) let $X = [X_1,...,X_k]$ be a\,vector of $k$ observations from our dataset. The joint probability distribution is then defined as:
$$p(X_1=x_1,...,X_k=x_k) = \prod_{j=1}^k \sum_{i=1}^n w_i\mathcal{N}(x_j|\mu_i,\Sigma_i)$$
Likelihood function is defined as:
$$L(w,\mu,\Sigma|X) = \prod_{j=1}^k \sum_{i=1}^n w_i\mathcal{N}(x_j|\mu_i,\Sigma_i)$$
$$\log(L(w,\mu,\Sigma|X)) = \sum_{j=1}^k \log(\sum_{i=1}^n w_i\mathcal{N}(x_j|\mu_i,\Sigma_i))$$
Unfortunately, differentiation and comparing to $0$ will not allow us to analytically solve this equation. In order to help with this, we will introduce a latent variables $z$ that explain which component generated given observation. Then:
$$p(X) = \sum_{i=1}^n p(x,z=i) = \sum_{i=1}^n p(z=i)p(x|z=i) = \sum_{i=1}^n w_i\mathcal{N}(\mu_i,\Sigma_i)$$
$$p(z_j=i_0|X_j) = \frac{p(X_j|z_j=i_0)p(z_j=i_0)}{p(X_j)} = 
\frac{w_{i_0}\mathcal{N}(\mu_{i_0},\Sigma_{i_0})}{\sum_{i=1}^n w_i\mathcal{N}(\mu_i,\Sigma_i)} = \gamma_{i_0}(z_j)$$
\subsection{EM algorithm}
We can notice that knowing either parameters or $\gamma$ allows us to compute the missing part. Furthermore, having a random guess about parameters, we can evaluate $\gamma$ probabilities and then estimate new parameters. Repeating this process, as well as measuring progress with log-likelihood, is a sketch of an iterative method known as the expectation–maximization (EM) algorithm. 
Let $\theta = (w,\mu,\Sigma)$, then:
$$Q(\theta, \theta_0) = \mathbb{E} \log p(X,Z|\theta) = \sum_z p(z|X,\theta^0)\log p(X,z|\theta)$$
We can simplify to this form:
$$Q(\theta, \theta_0) = \sum_z \gamma_(z) \log p(X,z|\theta)$$
And since:
$$\log p(X,Z|\theta) = \log \prod_{j=1}^k \prod_{i=1}^n w_i^{z_{i,j}} \mathcal{N}(X|\mu_i,\Sigma_i)^{z_{i,j}} = 
\sum_{j=1}^k \sum_{i=1}^n z_{i,j} (\log(w_i) + \log(\mathcal{N}(X|\mu_i,\Sigma_i)))$$
In summary, we have:
$$Q(\theta, \theta_0) = \sum_{j=1}^k \sum_{i=1}^n \gamma_i(z_j)(\log(w_i) + \log(\mathcal{N}(X|\mu_i,\Sigma_i)))$$
As for expectation phase (E), we evaluate $Q(\theta, \theta_0)$ given initial $\theta_0$. Maximization step (M) consists of solving $\theta_1 = \arg max_{\theta} Q(\theta, \theta_0)$. Solving for $Q$ is performed using Lagrange multipliers. These steps are repeated until the stop conditions are met. 
\begin{algorithm}
\caption{EM algorithm}\label{alg:cap}
\begin{algorithmic}
\Require $\theta, \gamma(Z)$
\State $\theta_0 \gets$ initial guess
\Repeat 
\State calculate $Q(\theta, \theta_0)$
\State $\theta_0 \gets \arg max_{\theta} Q(\theta, \theta_0)$
\Until{stop condition satisfied}
\end{algorithmic}
\end{algorithm}
\subsection{Bayesian Gaussian mixture models}
Finding parameters for GMM with EM algorithm does not include particular hyperparameters, as e.g. number of components. One can imagine that having $N$ observations and $N$ components in the form of Dirac delta functions would perfectly model a dataset, yet it would not be useful. The number of components can also be a very important parameter for the regularization of overfitting, a phenomenon in which the model may not be able to generalize outside of the training set. 
Bayesian interpretation allows us to use the prior probability distribution (Dirichlet distribution) to model the parameter space. Estimating the approximate posterior distribution over the parameters of a Gaussian mixture distribution yields the number of components from the dataset. 

\section{Wasserstein distance}
\begin{df}
Let $p$ and $q$ be two $d$-dimensional probability distributions and let $\Gamma(p,q)$ be a set of probability distributions whose marginals are $p$ and $q$ as the first and second factors respectively. Let $n\geq1$. The $n$-th Wasserstein distance $W_n$ between $p$ and $q$ is defined as:
$$W_n(p,q)=\inf_{\nu\in\Gamma(p,q)}(\int_{\mathbb{R}^d\times\mathbb{R}^d}||x-y||^nd\nu(x,y))^{1/n}=
\inf(\mathbb{E}(||X-Y||^n))^{1/n}$$
\end{df}
We may notice that looking at the measures corresponding to distributions $p$ and $q$, the set $\Gamma(p,q)$ is compact in the sense of weak convergence, therefore infimum is achievable. 
\begin{lm}
Let $\gamma\in\Gamma(p,q)$ and $x,y\in supp(\gamma)$ be $1$-dimensional ($d=1$) elements from support of $\gamma$ s.t. $x_1 < y_1$ and $y_2 < x_2$. Then the following inequality holds:
$$|x_1-y_2|^n + |y_1-x_2|^n < |x_1-x_2|^n + |y_1-y_2|^n$$
\end{lm}
\begin{proof}
We can notice that:
$$\exists t\in(0,1) \quad tx_1 + (1-t)x_2 = ty_1 + (1-t)y_2$$
Since points $x_1,y_2,tx_1+(1-t)x_2$ lie on the same line:
$$|x_1-y_2| = |x_1 - tx_1 - (1-t)x_2| + |tx_1 + (1-t)x_2 - y_2| = (1-t)|x_1-x_2| + t|y_1-y_2|$$
From Jensen's inequality we have:
$$|x_1-y_2|^n\leq (1-t)^n|x_1-x_2| + t^n|y_1-y_2|$$
Symmetrically we get the result: 
$$|x_2-y_1|^n\leq t^n|x_1-x_2| + (1-t)^n|y_1-y_2|$$
Summing gives us the inequality. 
\end{proof}
\begin{theorem}
Let $P,Q$ be cumulative distribution functions of distributions $p$ and $q$, by $P^{-1}$ and $Q^{-1}$ we mean inverse cumulative distribution functions or quantile functions; for $d=1$ we have:
$$W_n(p,q)=(\int_0^1|P^{-1}(t)-Q^{-1}(t)|^ndt)^{1/n}$$
\end{theorem}
\begin{proof}
If $\gamma\in\Gamma(p,q)$ satisfies infimum in the definition of Wasserstein distance, then $x_2\leq y_2$. Otherwise, by lemma, it would mean that there exists a better fit, where swapping $y_1$ and $y_2$ gives smaller value. \\
Let $x\in supp(p), y\in supp(q)$; then we can notice that $(x,y)\in supp(\gamma) \iff P(x) = Q(y)$. Indeed, $$(x,y)\in supp(\gamma) \iff \gamma(\mathbb{R},(-\infty, y])=\gamma((-\infty, x], (-\infty, y])=\gamma((-\infty, x], \mathbb{R})$$
Therefore we conclude that:
$$\int_{\mathbb{R}\times\mathbb{R}}|x-y|^n d\gamma(x,y) = \int_{supp(\gamma)}|x-y|^n d\gamma(x,y) = \int_0^1|P^{-1}(t)-Q^{-1}(t)|^ndt$$
\end{proof}
Thereafter, if not stated otherwise, we will consider $W := W_2$, i.e. Wasserstein distance for $n=2$. 
\subsection{Connections with transportation theory}
While considering probability as a mass over some space, Wasserstein distance realises the optimal transport problem for transforming one probability distribution into another. Suppose we have a cost function $c$ and probability distributions $p,q$. A transport plan is a function $\gamma$ s.t. $\gamma(x,y)$ is a volume of mass that needs to be moved from $x$ to $y$. Cost of a transport plan $\gamma$ is:
$$\int\int\gamma(x,y)c(x,y)dxdy = \int c(x,y)d\gamma(x,y)$$
Depending on the selection of the function $c$, going with infimum over possible plans yields us cost of optimal transport. 

\section{Wasserstein distance between two Gaussian mixture models}
In order to calculate Wasserstein distance between Gaussian mixture models, we would need to calculate an inverse cumulative distribution function for mixture of normal distributions. Since it is analytically impossible, a similar idea is adopted. 
\begin{theorem}\cite{delon}
$$W_2^2(\mathcal{N}(\mu_1,\Sigma_1), \mathcal{N}(\mu_2,\Sigma_2)) = ||\mu_1-\mu_2||^2 + tr(\Sigma_1 + \Sigma_2 -2(\Sigma_1^{1/2}\Sigma_2\Sigma_1^{1/2})^{1/2})$$
\end{theorem}
\begin{df}\footnote{Similar definition is proposed in \cite{delon}, but we only consider the finite case.}
Let $p_1=\sum_{i=1}^n w_{1i}\mathcal{N}(\mu_{1i},\Sigma_{1i}), p_2=\sum_{i=1}^m w_{2i}\mathcal{N}(\mu_{2i},\Sigma_{2i})$ be two GMMs. We define approximate Wasserstein distance between $p_1$ and $p_2$ in the following way:
$$\hat{W}(p_1,q_1) = min\{\sum_{i=1}^n\sum_{j=1}^m t_{ij}W(\mathcal{N}(\mu_{1i},\Sigma_{1i}),\mathcal{N}(\mu_{2j},\Sigma_{2j}))\;|\;{t_{i,j}\geq0:\forall i\sum_j t_{i,j}=w_{1i} \land\forall j\sum_i t_{i,j}=w_{2j}}\}$$
\end{df}
Proposed Wasserstein distance between mixture models is a straightforward extension of intuitions lying behind original Wasserstein distance. From the transport point of view, we are looking for the best assignment between corresponding mixtures. This can be extended to infinite dimensional form, i.e. when the number of components is not finite. The main difference is that here we do not seek the best transportation plan understood as a function or a measure, but a\,matrix of size $n\times m$. 
\subsection{Dual problem}
Let us consider a more general problem; let $a_i,b_j,c_{ij},d_{ij}$ be given nonnegative integers, the problem is following:
$$minimize\quad \sum_{i=1}^n\sum_{j=1}^m d_{ij}x_{ij}$$ 
$$subject\:to:\quad\forall i=1,...,n \sum_{j=1}^m x_{ij}=a_i\land\forall j=1,...,m \sum_{i=1}^n x_{ij}=b_j$$
$$\forall i,j\quad 0\leq x_{ij}\leq c_{ij}$$
In the case of Wasserstein distance $c_{ij}\equiv 1$, $d_{ij}$ is a Wasserstein distance between $i$-th component from first mixture and $j$-th component from second mixture, ${a_i}_{i=1}^n$ are weights of components from first mixture and ${b_j}_{j=1}^m$ are weights of components from second mixture. \\
The dual problem has the following form:
$$maximize\quad\sum_{i=1}^n a_i\alpha_i+\sum_{j=1}^m b_j\beta_j+\sum_{i=1}^n\sum_{j=1}^m c_{ij}\gamma_{ij}$$
$$subject\:to:\quad\forall i=1,...,n\forall j=1,...,m\quad d_{ij}\geq\alpha_i+\beta_j+\gamma_{ij}\land\gamma_{ij}\leq 0$$
It is worth to notice that the dual form immediately yields us a possible solution: setting all $\alpha_i=\beta_j=\gamma_{ij}=0$. 
\subsection{Solving with linear programming}
Finding Wasserstein distance between two mixture models comes down to solving particular transport problem. Therefore we can use the notations of graph theory: with a directed complete bipartite graph we have a cost over each edge being a Wasserstein distance between the given components, a capacity at each edge corresponding to the weight of the component and amount of flow, i.e. the value sought. We can use the network simplex algorithm to solve such a problem.
\subsection{GMM with Wasserstein distance as a classifier}
We present the algorithm for classification problem using Gaussian mixture models and Wasserstein distance. 
\begin{algorithm}
\caption{GMM with Wasserstein distance}\label{alg:gmm}
\begin{algorithmic}
\Require $\{(x_i,y_i)\}_{i=1}^k$ - training dataset with $m$ different classes
\Require $\{U_i\}_{i=1}^n$ - $n$ test datasets; within dataset each observation has the same label
\State 1. Split training dataset into $m$ sets $\{Z_j\}_{j=1}^,$ based on the label
\State 2. Fit Gaussian mixture model $p_j$ for each $Z_j$
\State 3. Fit Gaussian mixture model $q_i$ for each $U_i$
\State 4. Compute Wasserstein between $p_j$ and $q_i$ for each $i$, $j$
\State 5. Label the set $U_i$ with a label of the set $Z_{j_0}$, where $j_0 = argmin_j \hat{W}(p_j,q_i)$
\end{algorithmic}
\end{algorithm}

\section{Experiments}
\subsection{STL-10 dataset}
In the first experiment we used features extracted from an autoencoder neural network, which was trained in the case of image recognition. Original dataset is the STL-10 dataset\cite{coats}, which consists of total $13000$ images labeled as one of ten possible classes. Extracted representation has dimensionality of $512$, therefore during experiment we randomly choose some subset of dimensions. Received results has been compares with $L2$ distance and quadratic Jensen-Rényi divergence. 
\begin{figure}[H]
  \centering
  \includegraphics[width=1\textwidth]{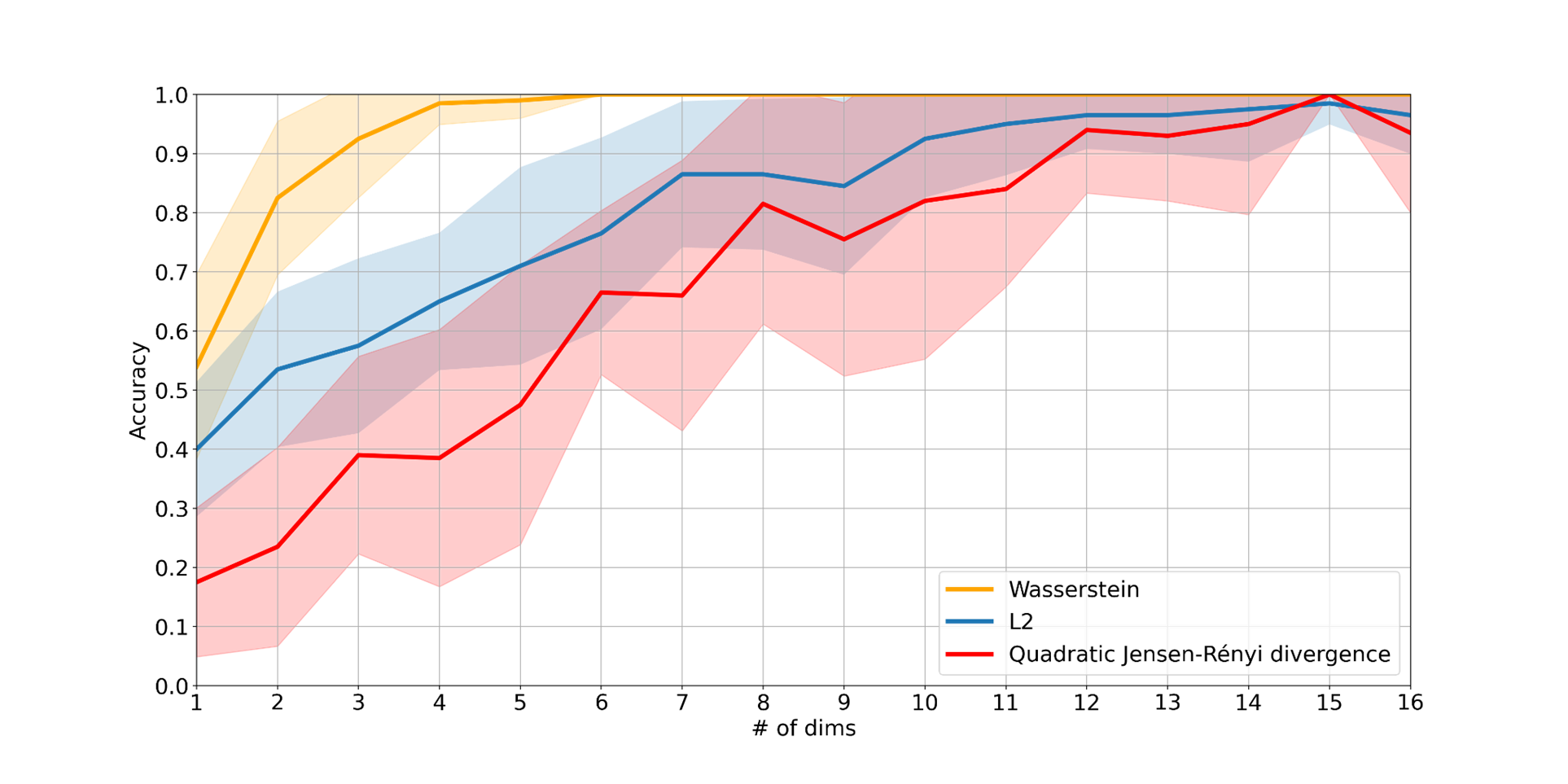}
  \caption{OX axis indicates number of selected components. At each step experiment was repeated $20$ times. Solid lines stands for mean results, while shaded area indicates $\pm1$ standard deviation from the mean. }
  \label{stl:stl}
\end{figure}
The task of matching data types using proposed method, i.e. Gaussian mixture models with Wasserstein distance, relies only on applied preprocessing. While autoencoder representations can be summaries themselves\cite{ja}, given the large volume of the data, our method is far more practical.
\subsection{Text data}
In the second experiment we operated on a large volume ($13000000$) of short text data, divided into chunks of examples with the same label. The task was to predict label of entire chunk. Preprocessing consists of transforming characters into features based on length and frequency of occurrence of given letters and signs. During experiment we performed $5$ $2$-fold cross-validations. Results were compared with a different approach (KNN algorithm) and similar approach with different distance function ($L2$).
\begin{figure}[H]
  \includegraphics[width=.7\textwidth]{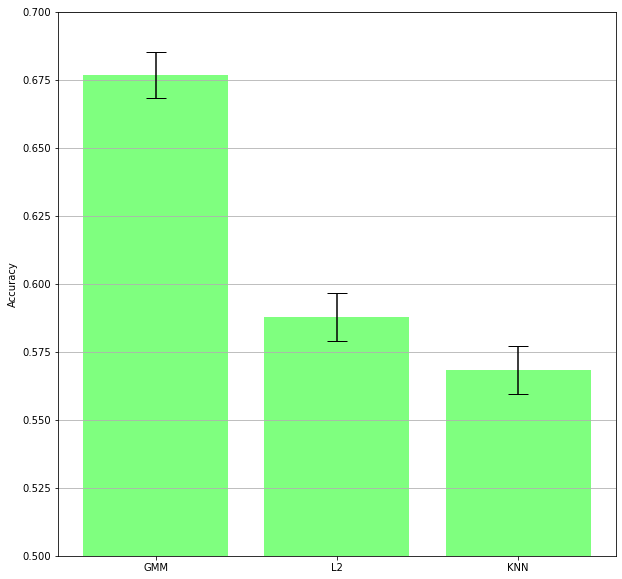}
  \centering
  \caption{Mean results from $5$ $2$-fold cross-validations. Standard deviations are indicated. }
  \label{ibi:ibi}
\end{figure}
Results conclude that proposed framework works better than compared methods.
\section{Conclusions and future work}
We derived the approximate easy easy-to-calculate version of the Wasserstein distance between Gaussian mixture models, that may find many applications in various fields of machine learning. In the case of big data, the greatest advantage is the avoidance of multiple calculations over the entire dataset, as the obtained summary allow for the estimation of similarity based only on the compacted data representations. \\
The future work may involve a statistical analysis of the properties extracted by Gaussian mixture models from a dataset, e.g. selecting important observations that may have had the greatest impact on the parameters. 

\bibliographystyle{unsrtnat}

\end{document}